\documentclass[letterpaper]{article} 
\usepackage{aaai2026}  
\usepackage{times}  
\usepackage{helvet}  
\usepackage{courier}  
\usepackage[hyphens]{url}  
\usepackage{graphicx} 
\usepackage{natbib}  
\usepackage{caption} 
\usepackage{pgfplots}
\pgfplotsset{compat=1.18} 
\usepackage{tikz}

\usepackage{multirow}

\usepackage{subcaption}

\usepackage[linesnumbered,ruled,vlined]{algorithm2e}

\usepackage[percent]{overpic}

\usepackage{booktabs}

\usepackage{tikz}

\usepackage{amsthm, amsmath, amsfonts}
\usepackage{soul}

\newtheorem{definition}{Definition}
\newtheorem{theorem}{Theorem}
\newtheorem{assumption}{Assumption}
\newtheorem{lemma}{Lemma}

\usepackage{hyperref} 
\nocopyright

\title{HFedATM: Hierarchical Federated Domain Generalization via\\Optimal Transport and Regularized Mean Aggregation}
\author{
    Thinh Nguyen\textsuperscript{\rm 1, 2},
    Trung Phan\textsuperscript{\rm 2},
    Binh T. Nguyen\textsuperscript{\rm 3},
    Khoa D Doan\textsuperscript{\rm 1, 2},
    Kok-Seng Wong\textsuperscript{\rm 1, 2}\footnote{Corresponding author.}
}
\affiliations{
    \textsuperscript{\rm 1}VinUni-Illinois Smart Health Center, VinUniversity, Hanoi, Vietnam\\
    \textsuperscript{\rm 2}College of Engineering \& Computer Science, VinUniversity, Hanoi, Vietnam\\
   \textsuperscript{\rm 3}Vietnam National University - Ho Chi Minh city, University of Science, Vietnam\\
    \{thinh.nth, 23trung.pl\}@vinuni.edu.vn, ngtbinh@hcmus.edu.vn, \{khoa.dd, wong.ks\}@vinuni.edu.vn
}

\usepackage{bibentry}

\begin{document}

\maketitle

\begin{abstract}
Federated Learning (FL) is a decentralized approach where multiple clients collaboratively train a shared global model without sharing their raw data. Despite its effectiveness, conventional FL faces scalability challenges due to excessive computational and communication demands placed on a single central server as the number of participating devices grows. Hierarchical Federated Learning (HFL) addresses these issues by distributing model aggregation tasks across intermediate nodes (stations), thereby enhancing system scalability and robustness against single points of failure. However, HFL still suffers from a critical yet often overlooked limitation: domain shift, where data distributions vary significantly across different clients and stations, reducing model performance on unseen target domains. While Federated Domain Generalization (FedDG) methods have emerged to improve robustness to domain shifts, their integration into HFL frameworks remains largely unexplored. In this paper, we formally introduce Hierarchical Federated Domain Generalization (HFedDG), a novel scenario designed to investigate domain shift within hierarchical architectures. Specifically, we propose HFedATM, a hierarchical aggregation method that first aligns the convolutional filters of models from different stations through Filter-wise Optimal Transport Alignment and subsequently merges aligned models using a Shrinkage-aware Regularized Mean Aggregation. Our extensive experimental evaluations demonstrate that HFedATM significantly boosts the performance of existing FedDG baselines across multiple datasets and maintains computational and communication efficiency. Moreover, theoretical analyses indicate that HFedATM achieves tighter generalization error bounds compared to standard hierarchical averaging, resulting in faster convergence and stable training behavior.
\end{abstract}


\section{Introduction}

Artificial intelligence (AI) has profoundly reshaped numerous sectors, from healthcare diagnostics to autonomous transportation, thanks to rapid advancements in machine learning~\cite{shang2024impact,baydoun2024artificial,atakishiyev2024explainable}. Nevertheless, real-world AI deployments must navigate persistent challenges, particularly data privacy, communication efficiency, and domain shift~\cite{mcmahan2017communication,kairouz2021advances}. As Federated Learning (FL) scales to a large number of clients, a single central server becomes a computational and communication bottleneck. To mitigate this, Hierarchical Federated Learning (HFL) shares aggregation responsibilities to intermediate layers, significantly reducing communication overhead and enhancing scalability~\cite{liu2020client}.

However, even HFL fails to address domain shift, where models trained in one domain perform poorly when deployed in unseen domains due to discrepancies between training and target distributions. Domain shifts significantly degrade the performance of the model, posing a major obstacle to the widespread adoption of FL models~\cite{fang2024hierarchical,li2023federated}. Although Federated Domain Generalization (FedDG) methods~\cite{nguyen2022fedsr,guo2023out, le2024efficiently} improve robustness to domain shifts, existing FedDG approaches have predominantly targeted single-server FL scenarios. Recent HFL variants like FedRC~\cite{guo2023fedrc} and MTGC~\cite{fang2024hierarchical} offer improved aggregation mechanisms, but these methods exchange intermediate statistics or multi-layer gradients, violating data-free privacy constraint inherent to domain generalization (DG), and lack theoretical guarantees ensuring optimal generalization. Thus, a research gap persists: \textit{the need for a data-free hierarchical aggregation method that efficiently optimizes generalization}.

Motivated by these critical shortcomings, we formally define a novel scenario named Hierarchical Federated Domain Generalization (HFedDG). We then present a theoretical error bound tailored for this scenario. Building upon this theory, we introduce \textbf{H}ierarchical \textbf{Fed}erated Optim\textbf{A}l \textbf{T}ransport and regularized \textbf{M}ean Aggregation (HFedATM) that leverages Filter-wise Optimal Transport (FOT) Alignment followed by a Shrinkage-aware Regularized Mean (RegMean) Aggregation. Unlike previous simplistic averaging approaches, HFedATM first aligns model weights across stations using FOT, ensuring semantically coherent aggregation, then optimally merges these aligned models through RegMean, providing a closed-form, theoretically justified solution that minimizes generalization error. Crucially, HFedATM remains completely data-free, preserving client privacy while simultaneously improving generalization capability. In summary, the main contributions are as follows.

\begin{itemize}
  \item We present the first formalization of the HFedDG scenario and derive a corresponding hierarchical error bound, revealing precisely how hierarchical aggregation affects DG performance.
  \item HFedATM, a method that combines FOT Alignment with Shrinkage-Aware RegMean Aggregation, theoretically guarantees a low overall DG constraint when local (client-station) training has minimal DG risk.
  \item Through extensive experiments across multiple benchmarks that span vision and natural language processing (NLP) tasks, we verify that HFedATM extensively enhances the accuracy of existing FedDG baselines and keeps training time overhead reasonable without leaking any local data.
\end{itemize}

\begin{figure*}[t]
\centering

\begin{overpic}[width=0.9\linewidth]{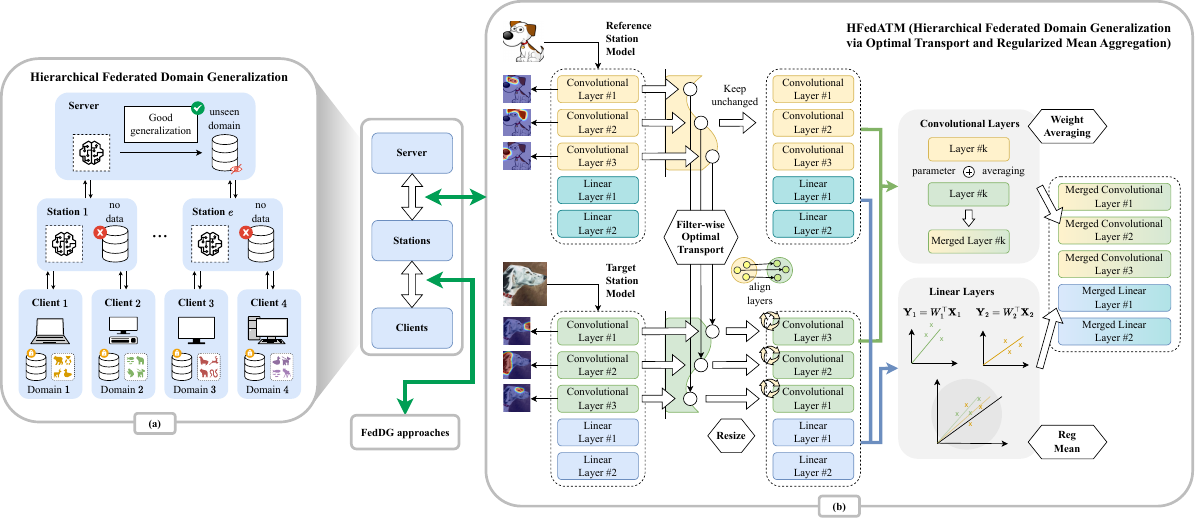}
        \put(2,5){\textbf{}\phantomsubcaption\label{fig:hfeddg_hfedatm_a}}
        \put(52,5){\textbf{}\phantomsubcaption\label{fig:hfeddg_hfedatm_b}}
    \end{overpic}
\caption{(a) Illustration of the HFedDG scenario, highlighting hierarchical model aggregation across clients, stations, and the server. (b) Overview of the proposed HFedATM methodology. Initially, convolutional filters from different station models are aligned via FOT alignment, resolving semantic discrepancies between filter indices. Subsequently, aligned convolutional layers are merged using weight averaging, while linear layers are aggregated through the RegMean procedure.}
\label{fig:hfeddg_hfedatm}
\end{figure*}

\section{Related Work}

\subsection{Hierarchical Federated Learning}

HFL introduces intermediate layers, called stations, between clients and the central server to distribute the aggregation workload, as shown in Figure~\ref{fig:hfeddg_hfedatm_a}. By aggregating client models at these intermediate layers, HFL significantly reduces communication bottlenecks, bandwidth demands, and latency issues inherent in centralized FL. Specifically, HFedAvg~\cite{liu2020client} extends traditional FedAvg by incorporating multi-layer aggregation (client-station-server), enhancing scalability and efficiency~\cite{liu2020client}.

However, HFedAvg and related methods often experience performance degradation under non-IID conditions~\cite{fang2024hierarchical}. To address this, methods such as MTGC~\cite{fang2024hierarchical} incorporate multi-timescale gradient correction, mitigating optimization drift due to heterogeneous data. Similarly, FedRC~\cite{guo2023fedrc} employs Gaussian-weighted model aggregation at the stations, effectively handling non-IID data distribution. Other approaches include EARA~\cite{abdellatif2022communication}, which minimizes the KL divergence between station distributions and a uniform distribution; FedHiSyn~\cite{li2022fedhisyn}, which facilitates inter-device weight sharing based on ring topology; and sFedHP~\cite{liu2023sparse}, leveraging hierarchical proximal mapping for personalized model adaptation. Despite these advances, current HFL methods mainly address heterogeneity among existing clients and overlook the crucial scenario of generalizing to unseen domains, which frequently occurs in real-world deployments and results in significant performance drops due to domain shift~\cite{oza2023unsupervised}.

\subsection{Federated Domain Generalization}

To address domain shift in FL, FedDG was first introduced by \citet{liu2021feddg}, proposing to learn domain-invariant representations among distributed clients. Subsequently, FedSR~\cite{nguyen2022fedsr} advanced this line by generating surrogate data through spectral transformations and randomized convolutions, effectively diversifying client distributions. Alternatively, regularization-based methods, represented by FedProx~\cite{li2020federated}, minimize local client overfitting by penalizing model divergence, thereby discouraging domain-specific features. Recent aggregation-focused approaches, notably FedIIR~\cite{guo2023out}, utilize invariant information regularization at the aggregation step to mitigate local biases and improve robustness to domain shifts. Furthermore, gPerXan~\cite{le2024efficiently} uses a personalized normalization layer together with a guiding regularizer to filter out domain-specific biases while promoting domain-invariant feature learning and improving generalization across unseen client domains.

However, these methods operate under the assumption of a single-server FL architecture, overlooking hierarchical federation settings. Thus, this paper is the first to formally extend FedDG into HFL, clearly defining the HFedDG scenario, and proposing an aggregation solution designed for data-free, privacy-preserving hierarchical model merging.

\section{Methodology}

\subsection{Problem Statement and Notation}

FedDG seeks a single model that can generalize to unseen domains, and HFedDG enriches this paradigm by inserting an intermediate layer of stations between the server and the clients. The hierarchical architecture progresses from \textbf{server} through \textbf{stations} to \textbf{clients}, trims bandwidth, enables domain‑aware clustering at the station, and allows partial participation without compromising stability. Each client belongs to exactly one training domain, and each training domain is represented at one or more stations. The target domains are sampled from a distribution that remains unseen. Formally, we have the following definition of HFedDG:

\begin{definition}
Let $E=\{1,\dots,N_E\}$ be the set of stations and $C_e=\{1,\dots,N_e\}$ the clients managed by the station $e$. In each round, a subset $\hat{C}_e\subseteq C_e$ of active clients is selected. Let $S_{e,i}=\{(x_{e,i}^{(j)},y_{e,i}^{(j)})\}_{j=1}^{n_{e,i}}\sim P^{e,i}_{XY}$ denote the client-level source domains, and these distributions are typically heterogeneous (domain shift). The stations do not hold or access raw data. Instead, they aggregate models trained by their associated clients. Let $P^{\star}_{XY}$ be an unseen target domain such that $P^{\star}_{XY}\neq P^{e,i}_{XY} \forall e,i$. The HFedDG task seeks a hypothesis $h$ that minimizes the expected loss $D_{\mathrm{target}}(h)=\mathbb{E}_{(x,y)\sim P^{\star}_{XY}} [ \ell(h(x),y) ]$ without exposing client data.
\end{definition}

\noindent Subsequently, we introduce a theorem that extends the DG bound established for single-server settings~\cite{li2023federated} to our framework. By decomposing the risk into client-level and server-level terms, it clarifies the remaining generalization gap once strong local learning has been achieved:
\begin{theorem}[Generalization‑error bound for HFedDG]
\label{theo:hfeddg}
Let $d_{\mathcal H}$ be the $\mathcal H$\nobreakdash-divergence. Given the client-level distributions $P_X^{e,i}$, let $\eta_e = \min_{i \in C_e} d_H(P_X^\star, P_X^{e,i})$ be the inner divergence and $\omega_e = \max_{i,j \in C_e} d_H(P_X^{e,i}, P_X^{e,j})$ be the breadth at each station $e$. Similarly, we define the server-level inner divergence and breadth as $\eta = \min_{e\in E, i\in C_e} d_H(P_X^\star, P_X^{e,i})$ and $\omega = \max_{(e,i),(e',j)} d_H(P_X^{e,i}, P_X^{e',j})$. Then, for any hypothesis $h\in\mathcal H$,
\begin{equation}
\label{eq:hfeddg_error_eq}
\begin{aligned}
D_{\mathrm{target}}(h) &\leq  
\sum_{e=1}^{N_E}\sum_{i=1}^{N_e}\rho_{e,i}^{\star}D_{e,i}(h) 
+ \frac{1}{2}\sum_{e=1}^{N_E}\rho_e^{\star}\omega_e\\
&+ \frac{1}{2}\omega + \sum_{e=1}^{N_E}\rho_e^{\star}\eta_e + \eta 
+ \lambda_\mathcal{H}(P_X^\star, P_X^\dagger).
\end{aligned}
\end{equation}
where $D_{e,i}(h)=\mathbb{E}_{(x,y)\sim P^{e,i}_{XY}} [\ell(h(x),y)]$ is client-level risk, $\rho_{e,i}^{\star}\geq 0$ and $\rho_e^{\star}=\sum_{i\in C_e}\rho_{e,i}^{\star}\geq 0$, with $\sum_e\rho_e^{\star}=1$, represent the optimal distributional proximity of the client distributions to the target distribution, $P_X^\dagger$ is the closest client-level distribution to $P_X^\star$, and $\lambda_{\mathcal{H}}$ is the joint-error term. 
\end{theorem}

After effective local training, client-level risks $D_{e,i}(h)$ become small, particularly when robust FedDG methods are used locally. In such a scenario, the generalization bound is influenced by the divergence terms $(\eta_{e},\omega_{e})$ and $(\eta,\omega)$, which characterize the intra-station and residual inter-station domain discrepancies, respectively.

\subsection{Overview of HFedATM}
\label{sec:overview}
After local DG has reduced the within-station shift, the remaining gap is the inter‑station divergence, driven by $(\eta, \omega)$ in the HFedDG error bound (Equation~\ref{eq:hfeddg_error_eq}). There are two main reasons for this gap:

\begin{itemize}
\item \textbf{Mismatched filter order.} The index of a convolutional filter is arbitrary: the first filter on the edge $A$ might be a vertical‑station detector, while the same concept is in the tenth filter on station $B$~\cite{singh2020model}. Index‑wise averaging, therefore, destroys semantics.
\item \textbf{Station‑specific correlations in linear layers.} Even if the filters were aligned, the linear layers that follow them encode correlations unique to each station's data. Therefore, naive averaging mixes incompatible statistics.
\end{itemize}
Our HFedATM tackles these two causes: \textit{(i) FOT Alignment} solves a Wasserstein assignment for every convolutional layer so that all stations share a common and semantically meaningful filter order; and \textit{(ii) RegMean Aggregation} then fuses every linear layer in closed form in this aligned feature space. Together, HFedATM collapses the inter-station gap. Figure~\ref{fig:hfeddg_hfedatm_b} gives an overview of the HFedATM workflow.

\subsection{Filter-wise Optimal Transport}
FOT Alignment removes the first mismatch before any weight averaging takes place. Specifically, it seeks a one-to-one permutation that brings the corresponding filters to the same index across all stations, so later merging operates on aligned channels. Let $W_{e}^{(l)}\in\mathbb{R}^{k\times c_{\text{in}}\times n\times n}$ denote the filter bank of station $e$ in layer $l$, where $k$ is the number of kernels and $n\ge1$ their spatial size. Each kernel is flattened and $\ell_2$‑normalized, producing vectors
\begin{equation}
\widetilde w_{e}^{(l)}[a]  = 
\frac{\mathrm{vec} \bigl(W_{e}^{(l)}[a]\bigr)}
     {\bigl\|\mathrm{vec} \bigl(W_{e}^{(l)}[a]\bigr)\bigr\|_{2}},
\qquad
a=1,\dots ,k .
\end{equation}
For a pair of stations $e$ and $e'$ we construct the squared Euclidean cost matrix (where $a,b=1,\dots ,k$):
\begin{equation}
D_{e,e'}^{(l)}(a,b)=
\bigl|\widetilde w_{e}^{(l)}[a]-\widetilde w_{e'}^{(l)}[b]\bigr|_{2}^{2}.
\end{equation}
Station $1$ is chosen as the reference; for every other station, we solve a discrete Optimal Transport problem
\begin{equation}
\Pi_{1,e}^{(l)} =
\arg\min_{\Pi\in\mathcal U}
\bigl\langle\Pi, D_{1,e}^{(l)}\bigr\rangle ,
\label{eq:fot}
\end{equation}
where $\mathcal U=\{\Pi\mid\Pi\mathbf 1=\mathbf 1, \Pi^{ \top} \mathbf 1=\mathbf 1, \Pi\ge0\}$ is the Birkhoff polytope. This Optimal Transport problem can be solved by the entropic-regularized Sinkhorn solver~\cite{cuturi2013sinkhorn}. The optimal permutation is then applied to the original filters:
\begin{equation}
W_{e}^{(l)} \leftarrow \Pi_{1,e}^{(l)} W_{e}^{(l)},
\end{equation}
ensuring filters at corresponding indices across stations capture the same visual primitives. After alignment, if filter sizes differ (e.g., $3\times3$ vs. $5\times5$), each filter is resized again to match the size of the reference station filter via bilinear interpolation~\cite{gonzalez2009digital}:
\begin{equation}
W_{e}^{(l)}[a]\leftarrow\text{resize}(W_{e}^{(l)}[a], n_{1}),\quad a=1,\dots,k.
\end{equation}

Since the optimal assignment cost depends only on the number of kernels $k$, and the resizing step is computationally inexpensive. After this stage, aligned stations produce coherent and consistent channels, establishing an efficient basis for the next aggregation step.

\subsection{Regularized Mean Aggregation}
The second stage of HFedATM merges the weights once the semantic correspondence has been enforced by FOT. Two weight types must be handled, including \textit{(i) aligned convolutional filters} and \textit{(ii) linear layers}.

After FOT, every filter index $a=1,\dots ,k$ refers to the same visual primitive on every station. Therefore, we can perform a weighted arithmetic mean as follows:
\begin{equation}
\overline W^{(l)}[a] =
\displaystyle\sum_{e=1}^{N_E}\gamma_{e}W_{e}^{(l)}[a] \Bigg/
\displaystyle\sum_{e=1}^{N_E}\gamma_{e}.
\label{eq:convmerge}
\end{equation}
We default to $\gamma_{e}=|\mathcal S_{e}|$ so that stations with more active clients contribute proportionally.

Regarding the linear layers, they depend on correlations between channels, and weight averaging them would ignore the activation geometry. Therefore, motivated by the work of~\citet{jin2022dataless},  During the very last forward pass of its local FedDG epoch, the client $\langle e,i\rangle$ records the activation matrix $X_{e,i}^{(l)}\in\mathbb R^{d\times m}$ for each dense layer $l$ ($d$=width, $m$=mini‑batch size) and forms the local Gram
\begin{equation}
\label{eq:localgram}
G_{e,i}^{(l)} = X_{e,i}^{(l)\top} X_{e,i}^{(l)}.
\end{equation}
The client optionally clips $\|G_{e,i}^{(l)}\|_{2}\le C$ and adds Gaussian noise for Differential-Privacy (DP), then uploads $G_{e,i}^{(l)}$ together with its weight update.

Because the station has no raw data, we simply average the incoming matrices
\begin{equation}
G_{e}^{(l)} =
\frac{1}{|\mathcal S_e|}\sum_{i\in\mathcal S_e} G_{e,i}^{(l)},
\end{equation}
and apply diagonal shrinkage
\begin{equation}
\widehat G_{e}^{(l)} =
\alpha G_{e}^{(l)} + (1-\alpha)\operatorname{diag}\bigl(G_{e}^{(l)}\bigr),
\quad 0\le\alpha\le1 ,
\label{eq:shrink}
\end{equation}
with $\alpha=0.75$ by default, followed by~\citet{jin2022dataless}. Only the shrunk matrix $\widehat G_{e}^{(l)}$ and the aligned weights $\widetilde W_{e}^{(l)}$ leave the station, and no data are exposed.

Given $\{\widehat G_{e}^{(l)},\widetilde W_{e}^{(l)}\}_{e=1}^{N_E}$, the server solves the objective that is data-free but aware of inter-station variability:
\begin{equation}
W_{\text{ATM}}^{(l)} =
\arg\min_{W}
\sum_{e=1}^{N_E}
\bigl| W^{\top} X_{e}^{(l)} - \widetilde W_{e}^{(l)\top} X_{e}^{(l)} \bigr|_{F}^{2}.
\end{equation}
Replacing $X_{e}^{(l)\top} X_{e}^{(l)}$ by $\widehat G_{e}^{(l)}$ yields the solution
\begin{equation}
W_{\text{ATM}}^{(l)} =
\Bigl( \sum_{e=1}^{N_E} \widehat G_{e}^{(l)} \Bigr)^{-1}
\Bigl( \sum_{e=1}^{N_E} \widehat G_{e}^{(l)} , \widetilde W_{e}^{(l)} \Bigr).
\label{eq:regmean}
\end{equation}
The complete algorithm is then displayed in Appendix~\ref{sec:algo}.

\begin{table*}[t]
\centering
{\small
\begin{tabular}{ccc|cccccccccccc|cccc}
\toprule
\multicolumn{3}{c|}{\textbf{Task ($\rightarrow$)}}                                                   & \multicolumn{12}{c|}{\textbf{Vision}}                                                                                                                                                                                                   & \multicolumn{4}{c}{\textbf{NLP}}                              \\ \midrule
\multicolumn{2}{c|}{\textbf{Baselines ($\downarrow$)}}                  & \multirow{2}{*}{$\lambda$} & \multicolumn{4}{c|}{\textbf{PACS}}                                                 & \multicolumn{4}{c|}{\textbf{Office‑Home}}                                          & \multicolumn{4}{c|}{\textbf{TerraInc}}                        & \multicolumn{4}{c}{\textbf{Amazon Reviews}}                   \\ \cmidrule{1-2} \cmidrule{4-19} 
\textbf{Client-Station}  & \multicolumn{1}{c|}{\textbf{Station-Server}} &                            & \textbf{P}    & \textbf{A}    & \textbf{C}    & \multicolumn{1}{c|}{\textbf{S}}    & \textbf{Pr}   & \textbf{Ar}   & \textbf{Cl}   & \multicolumn{1}{c|}{\textbf{R}}    & \textbf{L38}  & \textbf{L43}  & \textbf{L46}  & \textbf{L100} & \textbf{B}    & \textbf{D}    & \textbf{E}    & \textbf{K}    \\ \midrule
\multirow{2}{*}{FedAvg}  & \multicolumn{1}{c|}{+Avg}                    & \multirow{8}{*}{$1.0$}     & 81.8          & 77.7          & 78.0          & \multicolumn{1}{c|}{69.7}          & 64.3          & 52.1          & 45.1          & \multicolumn{1}{c|}{69.7}          & \textbf{45.7} & \textbf{48.4} & \textbf{40.1} & 34.7          & 70.3          & 70.1          & 69.8          & 69.5          \\
                         & \multicolumn{1}{c|}{+HFedATM}                &                            & \textbf{83.7} & \textbf{79.5} & \textbf{79.8} & \multicolumn{1}{c|}{\textbf{71.3}} & \textbf{65.7} & \textbf{53.7} & \textbf{46.9} & \multicolumn{1}{c|}{\textbf{71.2}} & 45.5          & 48.2          & 40.0          & \textbf{36.2} & \textbf{78.1} & \textbf{78.6} & \textbf{78.4} & \textbf{78.9} \\ \cmidrule{1-2} \cmidrule{4-19} 
\multirow{2}{*}{FedProx} & \multicolumn{1}{c|}{+Avg}                    &                            & 83.3          & 78.1          & 79.1          & \multicolumn{1}{c|}{69.4}          & 65.4          & 53.1          & \textbf{46.2} & \multicolumn{1}{c|}{70.6}          & 46.3          & \textbf{49.1} & 41.1          & \textbf{35.3} & 71.3          & 71.1          & 70.8          & 70.4          \\
                         & \multicolumn{1}{c|}{+HFedATM}                &                            & \textbf{84.9} & \textbf{79.9} & \textbf{80.8} & \multicolumn{1}{c|}{\textbf{71.1}} & \textbf{66.9} & \textbf{54.8} & 45.8          & \multicolumn{1}{c|}{\textbf{72.3}} & \textbf{47.9} & 48.9          & \textbf{43.1} & 35.2          & \textbf{79.1} & \textbf{79.5} & \textbf{79.3} & \textbf{79.7} \\ \cmidrule{1-2} \cmidrule{4-19} 
\multirow{2}{*}{FedSR}   & \multicolumn{1}{c|}{+Avg}                    &                            & 84.1          & 80.9          & 83.6          & \multicolumn{1}{c|}{73.4}          & 68.9          & 56.6          & 48.9          & \multicolumn{1}{c|}{74.3}          & 47.9          & 50.8          & 43.1          & 36.2          & 72.2          & 72.1          & 71.8          & 71.5          \\
                         & \multicolumn{1}{c|}{+HFedATM}                &                            & \textbf{87.7} & \textbf{84.4} & \textbf{86.6} & \multicolumn{1}{c|}{\textbf{76.7}} & \textbf{72.6} & \textbf{59.9} & \textbf{52.7} & \multicolumn{1}{c|}{\textbf{77.7}} & \textbf{51.3} & \textbf{54.6} & \textbf{46.4} & \textbf{39.7} & \textbf{80.9} & \textbf{81.2} & \textbf{81.1} & \textbf{81.6} \\ \cmidrule{1-2} \cmidrule{4-19} 
\multirow{2}{*}{FedIIR}  & \multicolumn{1}{c|}{+Avg}                    &                            & 85.1          & 81.6          & 84.1          & \multicolumn{1}{c|}{74.1}          & 69.4          & 57.3          & 49.7          & \multicolumn{1}{c|}{74.8}          & 48.3          & 51.6          & 43.4          & 36.8          & 72.6          & 72.4          & 72.2          & 71.8          \\
                         & \multicolumn{1}{c|}{+HFedATM}                &                            & \textbf{88.3} & \textbf{84.8} & \textbf{87.5} & \multicolumn{1}{c|}{\textbf{77.7}} & \textbf{73.5} & \textbf{60.7} & \textbf{53.5} & \multicolumn{1}{c|}{\textbf{78.7}} & \textbf{51.9} & \textbf{55.5} & \textbf{47.3} & \textbf{40.5} & \textbf{81.3} & \textbf{81.7} & \textbf{81.5} & \textbf{81.9} \\ \midrule
\multirow{2}{*}{FedAvg}  & \multicolumn{1}{c|}{+Avg}                    & \multirow{8}{*}{$0.1$}     & 79.5          & 75.3          & 74.9          & \multicolumn{1}{c|}{66.4}          & 61.5          & 49.1          & 42.7          & \multicolumn{1}{c|}{66.4}          & \textbf{43.1} & 46.2          & 38.1          & 32.1          & 68.2          & 68.0          & 67.7          & 67.4          \\
                         & \multicolumn{1}{c|}{+HFedATM}                &                            & \textbf{81.4} & \textbf{76.8} & \textbf{76.6} & \multicolumn{1}{c|}{\textbf{68.1}} & \textbf{63.4} & \textbf{50.7} & \textbf{44.6} & \multicolumn{1}{c|}{\textbf{68.2}} & 42.8          & \textbf{47.8} & \textbf{39.5} & \textbf{33.7} & \textbf{76.1} & \textbf{76.3} & \textbf{76.0} & \textbf{76.6} \\ \cmidrule{1-2} \cmidrule{4-19} 
\multirow{2}{*}{FedProx} & \multicolumn{1}{c|}{+Avg}                    &                            & 80.3          & 75.4          & 74.3          & \multicolumn{1}{c|}{65.7}          & 62.9          & 50.2          & 44.0          & \multicolumn{1}{c|}{67.0}          & 42.6          & 45.7          & 37.5          & \textbf{31.5} & 69.1          & 68.9          & 68.5          & 68.2          \\
                         & \multicolumn{1}{c|}{+HFedATM}                &                            & \textbf{81.9} & \textbf{77.3} & \textbf{76.2} & \multicolumn{1}{c|}{\textbf{67.5}} & \textbf{64.7} & \textbf{52.1} & \textbf{45.7} & \multicolumn{1}{c|}{\textbf{69.0}} & \textbf{44.3} & \textbf{47.6} & \textbf{39.2} & 31.2          & \textbf{77.0} & \textbf{77.3} & \textbf{77.0} & \textbf{77.5} \\ \cmidrule{1-2} \cmidrule{4-19} 
\multirow{2}{*}{FedSR}   & \multicolumn{1}{c|}{+Avg}                    &                            & 82.1          & 78.1          & 80.5          & \multicolumn{1}{c|}{69.3}          & 65.5          & 53.1          & 46.1          & \multicolumn{1}{c|}{70.1}          & 45.1          & 48.2          & \textbf{40.5} & 33.1          & 70.7          & 70.5          & 70.2          & 69.9          \\
                         & \multicolumn{1}{c|}{+HFedATM}                &                            & \textbf{85.6} & \textbf{81.8} & \textbf{83.7} & \multicolumn{1}{c|}{\textbf{72.9}} & \textbf{68.8} & \textbf{56.4} & \textbf{49.3} & \multicolumn{1}{c|}{\textbf{73.5}} & \textbf{48.6} & \textbf{51.9} & 44.1          & \textbf{33.9} & \textbf{78.3} & \textbf{78.1} & \textbf{77.7} & \textbf{77.4} \\ \cmidrule{1-2} \cmidrule{4-19} 
\multirow{2}{*}{FedIIR}  & \multicolumn{1}{c|}{+Avg}                    &                            & 83.0          & 79.0          & 81.1          & \multicolumn{1}{c|}{69.8}          & 66.0          & 53.5          & 46.5          & \multicolumn{1}{c|}{70.5}          & 45.0          & 48.3          & 40.6          & \textbf{33.6} & 71.0          & 70.7          & 70.4          & 70.1          \\
                         & \multicolumn{1}{c|}{+HFedATM}                &                            & \textbf{86.8} & \textbf{82.6} & \textbf{84.6} & \multicolumn{1}{c|}{\textbf{73.9}} & \textbf{70.0} & \textbf{57.2} & \textbf{50.1} & \multicolumn{1}{c|}{\textbf{74.3}} & \textbf{49.2} & \textbf{52.5} & \textbf{44.4} & 33.3          & \textbf{78.7} & \textbf{78.6} & \textbf{78.2} & \textbf{77.9} \\ \midrule
\multirow{2}{*}{FedAvg}  & \multicolumn{1}{c|}{+Avg}                    & \multirow{8}{*}{$0.0$}     & 70.5          & 64.6          & 65.5          & \multicolumn{1}{c|}{56.5}          & 53.6          & 41.0          & \textbf{36.1} & \multicolumn{1}{c|}{57.5}          & 37.1          & 39.4          & \textbf{32.2} & \textbf{27.1} & 64.1          & 64.0          & 63.7          & 63.5          \\
                         & \multicolumn{1}{c|}{+HFedATM}                &                            & \textbf{72.4} & \textbf{66.4} & \textbf{67.4} & \multicolumn{1}{c|}{\textbf{58.4}} & \textbf{55.1} & \textbf{42.7} & 35.9          & \multicolumn{1}{c|}{\textbf{59.2}} & \textbf{38.9} & \textbf{41.4} & 32.1          & 27.0          & \textbf{75.7} & \textbf{75.6} & \textbf{75.4} & \textbf{75.2} \\ \cmidrule{1-2} \cmidrule{4-19} 
\multirow{2}{*}{FedProx} & \multicolumn{1}{c|}{+Avg}                    &                            & 71.6          & 65.8          & 66.6          & \multicolumn{1}{c|}{57.1}          & 55.1          & 42.1          & \textbf{36.8} & \multicolumn{1}{c|}{58.2}          & 37.7          & 40.4          & \textbf{32.9} & \textbf{27.9} & 65.3          & 65.1          & 64.8          & 64.6          \\
                         & \multicolumn{1}{c|}{+HFedATM}                &                            & \textbf{73.4} & \textbf{67.8} & \textbf{68.3} & \multicolumn{1}{c|}{\textbf{59.0}} & \textbf{56.7} & \textbf{43.9} & 35.7          & \multicolumn{1}{c|}{\textbf{60.1}} & \textbf{39.9} & \textbf{42.3} & 32.5          & 26.8          & \textbf{76.3} & \textbf{76.2} & \textbf{75.9} & \textbf{75.8} \\ \cmidrule{1-2} \cmidrule{4-19} 
\multirow{2}{*}{FedSR}   & \multicolumn{1}{c|}{+Avg}                    &                            & 73.6          & 68.6          & 70.6          & \multicolumn{1}{c|}{59.3}          & 57.4          & 44.1          & 38.8          & \multicolumn{1}{c|}{60.5}          & 39.3          & 42.2          & 34.1          & \textbf{28.9} & 66.3          & 66.2          & 65.9          & 65.7          \\
                         & \multicolumn{1}{c|}{+HFedATM}                &                            & \textbf{77.4} & \textbf{72.5} & \textbf{74.4} & \multicolumn{1}{c|}{\textbf{62.9}} & \textbf{60.7} & \textbf{47.4} & \textbf{42.6} & \multicolumn{1}{c|}{\textbf{64.4}} & \textbf{43.3} & \textbf{46.1} & \textbf{38.1} & 28.5          & \textbf{78.7} & \textbf{78.6} & \textbf{78.3} & \textbf{78.2} \\ \cmidrule{1-2} \cmidrule{4-19} 
\multirow{2}{*}{FedIIR}  & \multicolumn{1}{c|}{+Avg}                    &                            & 76.1          & 71.1          & 73.1          & \multicolumn{1}{c|}{61.1}          & 59.3          & 46.3          & 40.6          & \multicolumn{1}{c|}{61.5}          & 40.3          & 43.3          & 35.4          & \textbf{29.6} & 66.8          & 66.6          & 66.3          & 66.1          \\
                         & \multicolumn{1}{c|}{+HFedATM}                &                            & \textbf{79.4} & \textbf{74.5} & \textbf{76.4} & \multicolumn{1}{c|}{\textbf{64.9}} & \textbf{62.7} & \textbf{49.4} & \textbf{43.8} & \multicolumn{1}{c|}{\textbf{65.3}} & \textbf{45.3} & \textbf{48.4} & \textbf{40.3} & 29.1          & \textbf{79.2} & \textbf{79.0} & \textbf{78.8} & \textbf{78.5} \\ \bottomrule
\end{tabular}
}
\caption{Performance (\%) comparison of FedDG baselines with and without HFedATM across multiple benchmarks and heterogeneity settings ($\lambda$). Bold values highlight the better-performing aggregation strategy within each baseline method.}
\label{tab:acc}
\end{table*}

\section{Theoretical Analysis}
We now provide a formal guarantee that HFedATM improves DG in the hierarchical setting. We state the main theorem below and defer all assumptions, supporting lemmas, and detailed proofs to Appendix~\ref{sec:theo_proof}.

\begin{theorem}
\label{theo:hfedatm_bound}
Assume the local training in every client achieves $D_{e,i}(h^{(R)})\le \varepsilon_{\mathrm{local}}$ for all $(e,i)$, and that the conditions of Theorem~\ref{theo:hfeddg} hold. Then the hypothesis $h^{(R)}_{\mathrm{ATM}}$ output by HFedATM satisfies
\begin{align}
D_{\mathrm{target}} \bigl(h^{(R)}_{\mathrm{ATM}}\bigr)
 \le &
\varepsilon_{\mathrm{local}}
+\frac{1}{2}(1-\beta)^R 
       ( 
           \omega^{(0)}
          +\sum_{e=1}^{N_E}\rho_e^{\star}\omega_e^{(0)}
         ) \notag \\
&+(1-\beta)^R 
       ( 
           \eta^{(0)}
          +\sum_{e=1}^{N_E}\rho_e^{\star}\eta_e^{(0)} \label{eq:hfedatm}
         )\\
&+\lambda_{\mathcal H}(P_X^{\star},P_X^{\dagger}), \notag
\end{align}
where $\{\rho_{e}^{\star}\}$ are the optimal coupling weights of Theorem~\ref{theo:hfeddg}.
\end{theorem}

The first line of the bound involves the \emph{client‑level risks}, which are already driven down by any strong FedDG method used during local training. The two subsequent lines contain the \emph{intra‑station} terms $(\eta_e,\omega_e)$ and the \emph{inter‑station} terms $(\eta,\omega)$. HFedATM's design attacks these divergence terms directly: \textbf{FOT Alignment} permutes convolutional filters so that semantically corresponding channels are matched across stations, reducing the breadth terms $\omega_e$ and~$\omega$; and \textbf{RegMean Aggregation} then merges the aligned weights in a data‑free form which provably minimizes the weighted sum of station losses, further tightening $\eta_e$ and~$\eta$. Consequently, as shown in Equation~\ref{eq:hfedatm}, $D_{\mathrm{target}}(h^{(R)}_{\mathrm{ATM}})$ is strictly smaller than that obtained by plain weight averaging, guaranteeing that HFedATM lowers error bound.

\section{Experiments}

\subsection{Experimental setup}

\paragraph{Baseline Methods} In client-station communication, we choose FedAvg~\cite{mcmahan2017communication} as the baseline. For directly addressing data heterogeneity, we also consider FedProx~\cite{li2020federated}. Next, we choose two FedDG methods, the regularization-based algorithm, FedSR~\cite{nguyen2022fedsr}, and the aggregation-based technique, FedIIR~\cite{guo2023out}. In station-server communication, we use weight averaging (Avg) as a baseline to evaluate its performance against HFedATM. We conducted all experiments using the framework provided by~\citet{tan2023pfedsim}.

\paragraph{Datasets} In this section, we evaluate our proposed HFedATM method for vision and NLP tasks. For vision classification tasks, we used three datasets, including \textbf{PACS}~\cite{li2017deeper}, \textbf{Office-Home}~\cite{venkateswara2017deep}, and \textbf{TerraInc}~\cite{beery2018recognition}. Regarding NLP tasks, we utilized \textbf{Amazon Reviews} dataset~\cite{balaji2018metareg}.

\paragraph{Data Partitioning} In our experiments, we simulate an HFL system with a total of $10$ stations and $100$ clients, and each station has $10$ clients; all clients will participate in training in each communication round. We use heterogeneous partitioning~\cite{bai2023benchmarking} to control client-level domain heterogeneity through the heterogeneity parameter $\lambda$, where $\lambda\in\left\{1.0, 0.1, 0.0\right\}$.

\paragraph{Model Architectures} Following previous works in FL~\cite{nguyen2022fedsr, guo2023out}, we used LeNet-5~\cite{lecun1998gradient} for vision tasks and RoBERTa-base~\cite{liu2019roberta} for NLP tasks. Detailed experimental setup is provided in Appendix~\ref{sec:detailed_exp}.

\subsection{Results}
\label{sec:results}

\begin{figure}[ht]
\centering
\includegraphics[width=\linewidth]{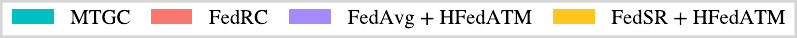}\\
\vspace*{0.3cm}
\includegraphics[width=0.9\linewidth]{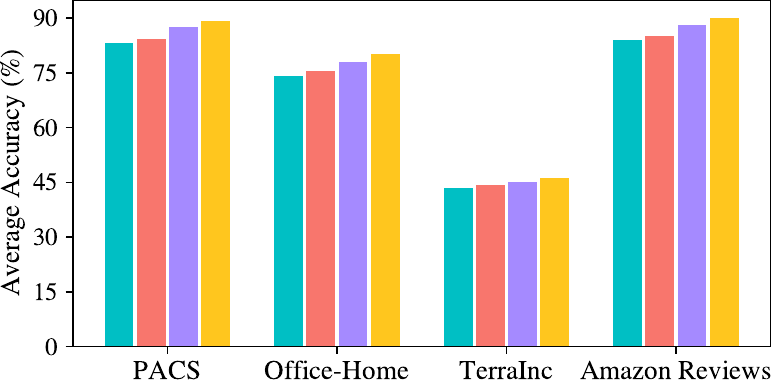}
\caption{Performance comparison of HFedATM (integrated with FedAvg and FedSR) against recent HFL methods, MTGC and FedRC, across diverse benchmarks. HFedATM consistently achieves higher accuracy.}
\label{fig:hfl_compare}
\end{figure}

\begin{table*}[t]
\centering
{\small
\begin{tabular}{c|c|cccccc|c}
\toprule
\multicolumn{2}{c|}{\textbf{Task ($\rightarrow$)}}     & \multicolumn{6}{c|}{\textbf{Vision}}                                                                                     & \textbf{NLP}     \\ \midrule
\multicolumn{2}{c|}{\textbf{Baselines ($\downarrow$)}} & \multicolumn{3}{c|}{\textbf{ResNet-18}}                               & \multicolumn{3}{c|}{\textbf{VGG-11}}             & \textbf{DeBERTa-base} \\ \midrule
\textbf{Client-Station}    & \textbf{Station-Server}   & PACS           & Office-Home    & \multicolumn{1}{c|}{TerraInc}       & PACS           & Office-Home    & TerraInc       & Amazon Reviews   \\ \midrule
\multirow{2}{*}{FedAvg}    & +Avg                      & 87.4          & 75.9          & \multicolumn{1}{c|}{49.8}          & 86.3          & 74.6          & 50.4          & 75.1            \\
                           & +HFedATM                  & \textbf{89.6} & \textbf{78.9} & \multicolumn{1}{c|}{\textbf{52.9}} & \textbf{88.3} & \textbf{77.6} & \textbf{53.6} & \textbf{78.7}   \\ \midrule
\multirow{2}{*}{FedProx}   & +Avg                      & 87.8          & 76.3          & \multicolumn{1}{c|}{50.2}          & 86.8          & 75.1          & 50.8          & 75.5            \\
                           & +HFedATM                  & \textbf{89.9} & \textbf{79.3} & \multicolumn{1}{c|}{\textbf{53.3}} & \textbf{88.7} & \textbf{77.9} & \textbf{54.0} & \textbf{79.1}   \\ \midrule
\multirow{2}{*}{FedSR}     & +Avg                     & 89.4          & 78.6          & \multicolumn{1}{c|}{52.7}          & 88.3          & 77.3          & 53.8          & 77.3            \\
                           & +HFedATM                & \textbf{92.8} & \textbf{81.9} & \multicolumn{1}{c|}{\textbf{56.5}} & \textbf{91.7} & \textbf{80.7} & \textbf{57.3} & \textbf{81.2}   \\ \midrule
\multirow{2}{*}{FedIIR}    & +Avg                     & 89.7          & 78.9          & \multicolumn{1}{c|}{53.1}          & 88.7          & 77.6          & 54.3          & 77.9            \\
                           & +HFedATM                & \textbf{93.0} & \textbf{82.2} & \multicolumn{1}{c|}{\textbf{56.9}} & \textbf{91.9} & \textbf{80.9} & \textbf{57.6} & \textbf{81.9}   \\ \bottomrule
\end{tabular}
}
\caption{Average accuracy ($\%$) comparison of FedDG methods with and without HFedATM across ResNet-18, VGG-11 (vision tasks), and DeBERTa-base (NLP task).}
\label{tab:acc_arch}
\end{table*}

\paragraph{Robustness across diverse tasks} We report the overall performance across four FedDG benchmarks in Table~\ref{tab:acc}. For vision tasks, FedAvg demonstrates limited generalization capability due to the absence of DG mechanisms. FedProx marginally improves performance over FedAvg by alleviating data heterogeneity, but its overall effectiveness remains constrained. In contrast, FedDG methods outperform FedAvg and FedProx due to their ability to induce more coherent client-level representations. Our HFedATM consistently enhances performance across these FedDG baselines, aligning closely with our theoretical insights from Theorem~\ref{theo:hfedatm_bound}. Specifically, it indicates that the generalization bound of HFedATM relies on the initial divergences $\eta^{(0)}$ and breadths $\omega^{(0)}$; when client-level models (FedAvg, FedProx) produce a lot of heterogeneous representations (e.g., $\lambda \leq 0.1$ for Office-Home and TerraInc), these terms remain large, limiting HFedATM's effectiveness. This empirical observation emphasizes HFedATM's dependency on strong client-level DG approaches to yield significant performance gains. Nevertheless, this scenario does not represent an inherent limitation since stronger client-side DG algorithms, such as FedSR or FedIIR, effectively reduce these initial divergences, showing HFedATM's benefits. Notably, for the NLP task (Amazon Reviews), HFedATM consistently outperforms baseline methods irrespective of heterogeneity, underscoring its robustness and general applicability across diverse modalities.

\paragraph{Robustness compared to other HFL methods} Recent HFL algorithms, MTGC~\cite{fang2024hierarchical} and FedRC~\cite{kou2024fedrc}, address data heterogeneity but either involve heavy computation (MTGC) or exchange distributional moments that risk privacy leakage (FedRC). Figure~\ref{fig:hfl_compare} benchmarks these algorithms against HFedATM on four datasets. Our method achieves the highest accuracy on both datasets while remaining data-free and computationally lightweight.

\begin{table}[ht]
\centering

{\small
\begin{tabular}{c|l|cc|c}
\toprule
\multicolumn{1}{c|}{\multirow{2}{*}{\textbf{Dataset}}} & \multirow{2}{*}{\textbf{Method}} & \multicolumn{2}{c|}{\textbf{Training Time (s)}} & \multirow{2}{*}{\textbf{Increase (\%)}} \\ \cmidrule{3-4}
                                  &                                   & \textbf{$+\textbf{Avg}$} & \textbf{$+\textbf{HFedATM}$} &                                                      \\ \midrule
\multirow{4}{*}{PACS}             
                                  & FedAvg                            & 20.4             & 22.1             & 8.4                                                 \\
                                  & FedProx                           & 21.3             & 23.2             & 8.8                                                 \\
                                  & FedSR                             & 22.6             & 24.4             & 7.9                                                 \\
                                  & FedIIR                            & 22.1             & 23.9             & 8.2                                                 \\ \midrule
\multirow{4}{*}{Office-Home}      
                                  & FedAvg                            & 35.7             & 37.3             & 4.7                                                 \\
                                  & FedProx                           & 36.4             & 38.2             & 5.0                                                 \\
                                  & FedSR                             & 37.3             & 39.1             & 4.9                                                 \\
                                  & FedIIR                            & 37.4             & 39.2             & 4.9                                                 \\ \midrule
\multirow{4}{*}{TerraInc}   
                                  & FedAvg                            & 39.5             & 41.3             & 4.5                                                 \\
                                  & FedProx                           & 40.6             & 42.4             & 4.5                                                 \\
                                  & FedSR                             & 41.9             & 43.8             & 4.3                                                 \\
                                  & FedIIR                            & 41.5             & 43.4             & 4.6                                                 \\ \midrule
\multirow{4}{*}{\begin{tabular}[c]{@{}c@{}}Amazon \\ Review\end{tabular}}   
                                  & FedAvg                            & 60.4             & 62.2             & 2.8                                                 \\
                                  & FedProx                           & 61.4             & 63.3             & 3.1                                                 \\
                                  & FedSR                             & 62.7             & 64.5             & 2.9                                                 \\
                                  & FedIIR                            & 62.2             & 64.0             & 2.9                                                 \\ \bottomrule
\end{tabular}
}
\caption{Training time (in seconds) of various FedDG approaches, with and without attached HFedATM, per communication round, averaged over unseen domains.}
\label{tab:latency}
\end{table}

\begin{table}[ht]
\centering
{\small
\begin{tabular}{l|cccc}
\toprule
\multicolumn{1}{c|}{\multirow{1}{*}{\textbf{Variant}}}            & \textbf{PACS} & \textbf{Office-Home} & \textbf{TerraInc} & \begin{tabular}[c]{@{}c@{}}\textbf{Amazon} \\ \textbf{Reviews}\end{tabular} \\ \midrule
\multicolumn{5}{c}{\textbf{FedAvg + HFedATM}}                                                \\ \midrule
w/o FOT                      & 76.9          & 58.7                & 41.1                   &  72.2                  \\
w/o RegMean                  & 77.7          & 59.1                & 42.4                   & 75.4                   \\
\textbf{Full Method}         & \textbf{78.3} & \textbf{60.5}       & \textbf{43.2}          & \textbf{78.5}          \\ \midrule
\multicolumn{5}{c}{\textbf{FedProx + HFedATM}}                                               \\ \midrule
w/o FOT                      & 75.6          & 61.4                & 42.1                   & 76.3                   \\
w/o RegMean                  & 78.0          & 61.7                & 41.7                   & 74.2                   \\
\textbf{Full Method}         & \textbf{79.0} & \textbf{62.3}       & \textbf{44.2}          & \textbf{79.3}          \\ \midrule
\multicolumn{5}{c}{\textbf{FedSR + HFedATM}}                                                 \\ \midrule
w/o FOT                      & 80.2          & 63.2                & 44.2                   & 73.0                   \\
w/o RegMean                  & 79.7          & 65.9                & 45.3                   & 76.1                   \\
\textbf{Full Method}         & \textbf{81.3} & \textbf{67.7}       & \textbf{47.5}          & \textbf{81.0}          \\ \midrule
\multicolumn{5}{c}{\textbf{FedIIR + HFedATM}}                                                \\ \midrule
w/o FOT                      & 78.5          & 65.7                & 44.6                   & 78.3                   \\
w/o RegMean                  & 77.4          & 65.8                & 45.8                   & 77.3                   \\
\textbf{Full Method}         & \textbf{82.2} & \textbf{67.6}       & \textbf{48.8}          & \textbf{81.6}          \\ \bottomrule
\end{tabular}
}
\caption{Ablation study of HFedATM components. Results for the full HFedATM method are obtained with $\lambda = 1.0$.}
\label{tab:ablation}
\end{table}

\paragraph{Robustness across different architectures}
To confirm the general applicability of HFedATM, we evaluated its robustness across diverse architectures. Specifically, we conducted additional experiments using ResNet-18~\cite{he2016deep} and VGG-11~\cite{simonyan2014very} for vision tasks (PACS, Office-Home, and TerraInc) and DeBERTa-base~\cite{he2020deberta} for NLP tasks (Amazon Reviews). As summarized in Table~\ref{tab:acc_arch}, FedDG methods consistently show better performance when integrated with HFedATM across these architectures and datasets, demonstrating their effectiveness regardless of network structure.

\begin{figure}[ht]
\centering
\includegraphics[width=0.9\linewidth]{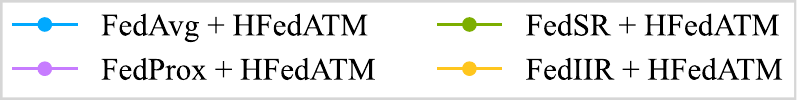}\\
\vspace*{0.2cm}
\includegraphics[width=0.96\linewidth]{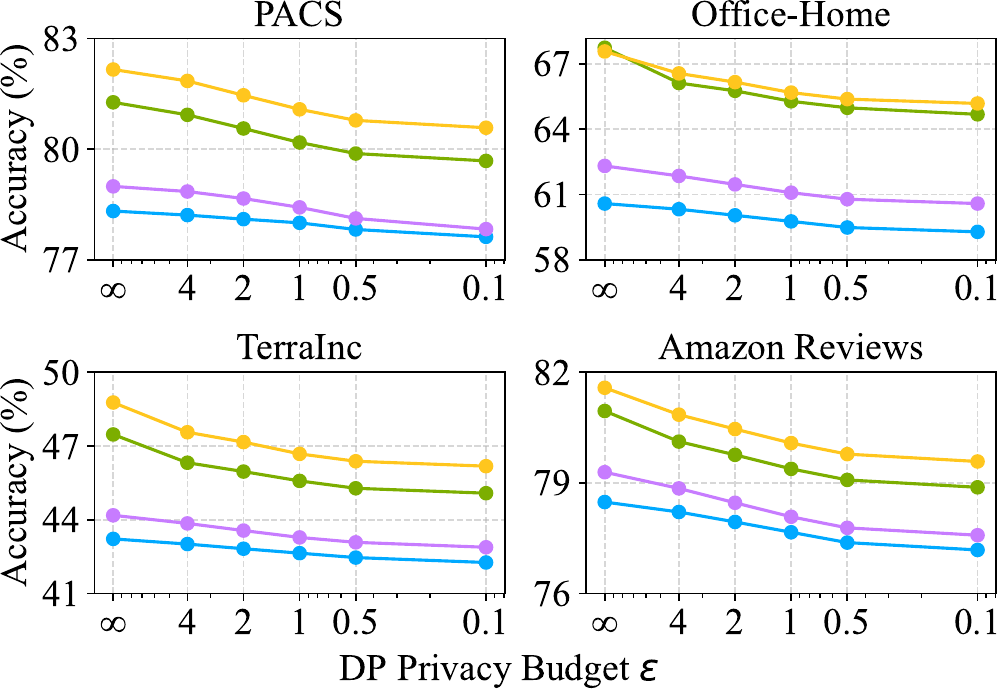}
\caption{Effect of varying DP budgets on the accuracy of HFedATM integrated with different baselines. Results show that HFedATM consistently maintains stable and robust performance even under various privacy constraints.}
\label{fig:privacy}
\end{figure}

\paragraph{Training latency} In practical scenarios such as smartphone apps~\cite{shin2023fedtherapist} and smart surveillance systems~\cite{pang2023federated}, rapid communication between servers and client devices is critical. We thus measure the wall-clock training latency (time per round) across all four datasets. As shown in Table~\ref{tab:latency}, HFedATM incurs a moderate latency overhead of less than $10\%$ per round while consistently providing improved accuracy. This trade-off is primarily driven by additional computations involved in our aggregation method, including the FOT Alignment and RegMean procedures. In particular, the overhead remains minimal due to the computational efficiency of FOT Alignment through the Sinkhorn algorithm~\cite{cuturi2013sinkhorn}, which requires only tens of milliseconds per convolutional layer on standard CPUs. Additionally, the Gram matrices required for RegMean can be computed efficiently in a single forward pass during the final station epoch before server aggregation, further controlling computational demands. Thus, HFedATM effectively balances a slight increase in training latency with extensive gains in DG performance.

\paragraph{Privacy robustness} Sharing Gram matrices, which capture second-order feature correlations, poses potential privacy risks. To address this, we evaluate the robustness of HFedATM by injecting Gaussian noise into Gram matrices to ensure $(\varepsilon, \delta)$-DP, with privacy budgets $\varepsilon \in \{4, 2, 1, 0.5, 0.1\}$ (where $\varepsilon = \infty$ denotes the nonprivate scenario). Figure~\ref{fig:privacy} illustrates the performance in multiple datasets, including PACS, Office-Home, TerraInc, and Amazon Reviews. HFedATM consistently maintains stable performance with moderate DP noise ($\varepsilon \ge 1$), demonstrating an accuracy drop of less than $2\%$ on all datasets. Even under strict privacy guarantees ($\varepsilon = 0.1$), HFedATM exhibits a graceful degradation and still outperforms baselines across vision and NLP tasks.

\section{Ablation Study}

To analyze the importance of each component in HFedATM, we performed ablation studies on four datasets. Table~\ref{tab:ablation} presents the accuracy results when disabling FOT Alignment or RegMean individually and also simultaneously. We find that omitting either component reduces accuracy, while removing both leads to a larger drop. This indicates that FOT and RegMean play complementary roles, each addressing distinct aspects of DG.

\section{Limitations}

\paragraph{Different Client Architectures} A straightforward limitation of HFedATM is its reliance on homogeneous client models, i.e., it assumes all clients and stations train identical architectures. However, this assumption aligns with existing literature in HFL, where model homogeneity is standard practice to facilitate effective hierarchical aggregation~\cite{liu2020client, fang2024hierarchical, guo2023fedrc}. Importantly, in this common scenario, HFedATM extensively increases DG capability compared to naive averaging. Exploring adaptations to accommodate heterogeneous models remains a compelling future direction, as the real-world FL setting frequently involves diverse device capabilities and thus heterogeneous architectures~\cite{ye2023heterogeneous}.

\paragraph{Data Privacy} The computation of inner product matrices required by RegMean raises potential privacy concerns. Specifically, Gram matrices ($G = X^\top X$) could reveal sensitive information if attackers attempt inversion to reconstruct client-side data. However, we show in Appendix~\ref{sec:privacy_analysis} that Gram matrices inherently possess a fundamental property of non-invertibility. Thus, the privacy risk associated with transmitting Gram matrices in HFedATM is substantially mitigated, ensuring a favorable trade-off between privacy preservation and model robustness.

\section{Conclusion}

This paper first explores the implementation of FedDG in the HFL setting. To effectively tackle domain shift, we proposed HFedATM, a data-free hierarchical aggregation method that combines FOT alignment and RegMean Aggregation. Theoretical analysis revealed that HFedATM provides a tighter generalization-error bound, highlighting how this hierarchical aggregation impacts DG performance. Empirical results in multiple benchmark datasets demonstrated that HFedATM substantially enhances existing FedDG methods, particularly when robust DG methods are already employed in the training stage of the client station. Moreover, we confirmed the efficiency of HFedATM in terms of communication latency and its robustness against privacy constraints, showing only minimal accuracy degradation under strong DP guarantees. Future work may address current limitations by exploring model adaptation techniques to enable the aggregation of heterogeneous client architectures and further improve privacy guarantees to mitigate potential risks associated with sharing intermediate statistics.


\bibliography{main}

\begin{thebibliography}{38}
\providecommand{\natexlab}[1]{#1}

\bibitem[{Abdellatif et~al.(2022)Abdellatif, Mhaisen, Mohamed, Erbad, Guizani, Dawy, and Nasreddine}]{abdellatif2022communication}
Abdellatif, A.~A.; Mhaisen, N.; Mohamed, A.; Erbad, A.; Guizani, M.; Dawy, Z.; and Nasreddine, W. 2022.
\newblock Communication-efficient hierarchical federated learning for IoT heterogeneous systems with imbalanced data.
\newblock \emph{Future Generation Computer Systems}, 128: 406--419.

\bibitem[{Atakishiyev et~al.(2024)Atakishiyev, Salameh, Yao, and Goebel}]{atakishiyev2024explainable}
Atakishiyev, S.; Salameh, M.; Yao, H.; and Goebel, R. 2024.
\newblock Explainable artificial intelligence for autonomous driving: A comprehensive overview and field guide for future research directions.
\newblock \emph{IEEE Access}.

\bibitem[{Bai, Bagchi, and Inouye(2023)}]{bai2023benchmarking}
Bai, R.; Bagchi, S.; and Inouye, D.~I. 2023.
\newblock Benchmarking algorithms for federated domain generalization.
\newblock \emph{arXiv preprint arXiv:2307.04942}.

\bibitem[{Balaji, Sankaranarayanan, and Chellappa(2018)}]{balaji2018metareg}
Balaji, Y.; Sankaranarayanan, S.; and Chellappa, R. 2018.
\newblock Metareg: Towards domain generalization using meta-regularization.
\newblock \emph{Advances in neural information processing systems}, 31.

\bibitem[{Baydoun et~al.(2024)Baydoun, Jia, Zaorsky, Kashani, Rao, Shoag, Vince~Jr, Bittencourt, Zuhour, Price et~al.}]{baydoun2024artificial}
Baydoun, A.; Jia, A.~Y.; Zaorsky, N.~G.; Kashani, R.; Rao, S.; Shoag, J.~E.; Vince~Jr, R.~A.; Bittencourt, L.~K.; Zuhour, R.; Price, A.~T.; et~al. 2024.
\newblock Artificial intelligence applications in prostate cancer.
\newblock \emph{Prostate cancer and prostatic diseases}, 27(1): 37--45.

\bibitem[{Beery, Van~Horn, and Perona(2018)}]{beery2018recognition}
Beery, S.; Van~Horn, G.; and Perona, P. 2018.
\newblock Recognition in terra incognita.
\newblock In \emph{Proceedings of the European conference on computer vision (ECCV)}, 456--473.

\bibitem[{Ben-David et~al.(2010)Ben-David, Blitzer, Crammer, Kulesza, Pereira, and Vaughan}]{ben2010theory}
Ben-David, S.; Blitzer, J.; Crammer, K.; Kulesza, A.; Pereira, F.; and Vaughan, J.~W. 2010.
\newblock A theory of learning from different domains.
\newblock \emph{Machine learning}, 79(1): 151--175.

\bibitem[{Cuturi(2013)}]{cuturi2013sinkhorn}
Cuturi, M. 2013.
\newblock Sinkhorn Distances: Lightspeed Computation of Optimal Transport.
\newblock In Burges, C.; Bottou, L.; Welling, M.; Ghahramani, Z.; and Weinberger, K., eds., \emph{Advances in Neural Information Processing Systems}, volume~26. Curran Associates, Inc.

\bibitem[{Fang et~al.(2024)Fang, Han, Chen, Wang, and Brinton}]{fang2024hierarchical}
Fang, W.; Han, D.-J.; Chen, E.; Wang, S.; and Brinton, C. 2024.
\newblock Hierarchical federated learning with multi-timescale gradient correction.
\newblock \emph{Advances in Neural Information Processing Systems}, 37: 78863--78904.

\bibitem[{Gonzalez(2009)}]{gonzalez2009digital}
Gonzalez, R.~C. 2009.
\newblock \emph{Digital image processing}.
\newblock Pearson education india.

\bibitem[{Guo et~al.(2023)Guo, Guo, Cao, Wu, and Chang}]{guo2023out}
Guo, Y.; Guo, K.; Cao, X.; Wu, T.; and Chang, Y. 2023.
\newblock Out-of-distribution generalization of federated learning via implicit invariant relationships.
\newblock In \emph{International Conference on Machine Learning}, 11905--11933. PMLR.

\bibitem[{Guo, Tang, and Lin(2023)}]{guo2023fedrc}
Guo, Y.; Tang, X.; and Lin, T. 2023.
\newblock Fedrc: Tackling diverse distribution shifts challenge in federated learning by robust clustering.
\newblock \emph{arXiv preprint arXiv:2301.12379}.

\bibitem[{He et~al.(2016)He, Zhang, Ren, and Sun}]{he2016deep}
He, K.; Zhang, X.; Ren, S.; and Sun, J. 2016.
\newblock Deep residual learning for image recognition.
\newblock In \emph{Proceedings of the IEEE conference on computer vision and pattern recognition}, 770--778.

\bibitem[{He et~al.(2020)He, Liu, Gao, and Chen}]{he2020deberta}
He, P.; Liu, X.; Gao, J.; and Chen, W. 2020.
\newblock Deberta: Decoding-enhanced bert with disentangled attention.
\newblock \emph{arXiv preprint arXiv:2006.03654}.

\bibitem[{Jin et~al.(2022)Jin, Ren, Preotiuc-Pietro, and Cheng}]{jin2022dataless}
Jin, X.; Ren, X.; Preotiuc-Pietro, D.; and Cheng, P. 2022.
\newblock Dataless knowledge fusion by merging weights of language models.
\newblock \emph{arXiv preprint arXiv:2212.09849}.

\bibitem[{Kairouz et~al.(2021)Kairouz, McMahan, Avent, Bellet, Bennis, Bhagoji, Bonawitz, Charles, Cormode, Cummings et~al.}]{kairouz2021advances}
Kairouz, P.; McMahan, H.~B.; Avent, B.; Bellet, A.; Bennis, M.; Bhagoji, A.~N.; Bonawitz, K.; Charles, Z.; Cormode, G.; Cummings, R.; et~al. 2021.
\newblock Advances and open problems in federated learning.
\newblock \emph{Foundations and trends{\textregistered} in machine learning}, 14(1--2): 1--210.

\bibitem[{Kou et~al.(2024)Kou, Lin, Tang, Wang, Zhu, and Wu}]{kou2024fedrc}
Kou, W.-B.; Lin, Q.; Tang, M.; Wang, S.; Zhu, G.; and Wu, Y.-C. 2024.
\newblock Fedrc: A rapid-converged hierarchical federated learning framework in street scene semantic understanding.
\newblock In \emph{2024 IEEE/RSJ International Conference on Intelligent Robots and Systems (IROS)}, 2578--2585. IEEE.

\bibitem[{Le et~al.(2024)Le, Ho, Do, Le-Phuoc, and Wong}]{le2024efficiently}
Le, K.; Ho, L.; Do, C.; Le-Phuoc, D.; and Wong, K.-S. 2024.
\newblock Efficiently Assemble Normalization Layers and Regularization for Federated Domain Generalization.
\newblock In \emph{Proceedings of the IEEE/CVF Conference on Computer Vision and Pattern Recognition}, 6027--6036.

\bibitem[{LeCun et~al.(1998)LeCun, Bottou, Bengio, and Haffner}]{lecun1998gradient}
LeCun, Y.; Bottou, L.; Bengio, Y.; and Haffner, P. 1998.
\newblock Gradient-based learning applied to document recognition.
\newblock \emph{Proceedings of the IEEE}, 86(11): 2278--2324.

\bibitem[{Li et~al.(2017)Li, Yang, Song, and Hospedales}]{li2017deeper}
Li, D.; Yang, Y.; Song, Y.-Z.; and Hospedales, T.~M. 2017.
\newblock Deeper, broader and artier domain generalization.
\newblock In \emph{Proceedings of the IEEE international conference on computer vision}, 5542--5550.

\bibitem[{Li et~al.(2022)Li, Hu, Zhang, Liu, Yin, Peng, and Dou}]{li2022fedhisyn}
Li, G.; Hu, Y.; Zhang, M.; Liu, J.; Yin, Q.; Peng, Y.; and Dou, D. 2022.
\newblock FedHiSyn: A hierarchical synchronous federated learning framework for resource and data heterogeneity.
\newblock In \emph{Proceedings of the 51st International Conference on Parallel Processing}, 1--11.

\bibitem[{Li et~al.(2020)Li, Sahu, Zaheer, Sanjabi, Talwalkar, and Smith}]{li2020federated}
Li, T.; Sahu, A.~K.; Zaheer, M.; Sanjabi, M.; Talwalkar, A.; and Smith, V. 2020.
\newblock Federated optimization in heterogeneous networks.
\newblock \emph{Proceedings of Machine learning and systems}, 2: 429--450.

\bibitem[{Li et~al.(2023)Li, Wang, Zeng, Donta, Murturi, Huang, and Dustdar}]{li2023federated}
Li, Y.; Wang, X.; Zeng, R.; Donta, P.~K.; Murturi, I.; Huang, M.; and Dustdar, S. 2023.
\newblock Federated domain generalization: A survey.
\newblock \emph{arXiv preprint arXiv:2306.01334}.

\bibitem[{Liu et~al.(2020)Liu, Zhang, Song, and Letaief}]{liu2020client}
Liu, L.; Zhang, J.; Song, S.; and Letaief, K.~B. 2020.
\newblock Client-edge-cloud hierarchical federated learning.
\newblock In \emph{ICC 2020-2020 IEEE international conference on communications (ICC)}, 1--6. IEEE.

\bibitem[{Liu et~al.(2021)Liu, Chen, Qin, Dou, and Heng}]{liu2021feddg}
Liu, Q.; Chen, C.; Qin, J.; Dou, Q.; and Heng, P.-A. 2021.
\newblock Feddg: Federated domain generalization on medical image segmentation via episodic learning in continuous frequency space.
\newblock In \emph{Proceedings of the IEEE/CVF conference on computer vision and pattern recognition}, 1013--1023.

\bibitem[{Liu et~al.(2023)Liu, Wang, Shao, and Li}]{liu2023sparse}
Liu, X.; Wang, Q.; Shao, Y.; and Li, Y. 2023.
\newblock Sparse federated learning with hierarchical personalization models.
\newblock \emph{IEEE Internet of Things Journal}.

\bibitem[{Liu et~al.(2019)Liu, Ott, Goyal, Du, Joshi, Chen, Levy, Lewis, Zettlemoyer, and Stoyanov}]{liu2019roberta}
Liu, Y.; Ott, M.; Goyal, N.; Du, J.; Joshi, M.; Chen, D.; Levy, O.; Lewis, M.; Zettlemoyer, L.; and Stoyanov, V. 2019.
\newblock Roberta: A robustly optimized bert pretraining approach.
\newblock \emph{arXiv preprint arXiv:1907.11692}.

\bibitem[{McMahan et~al.(2017)McMahan, Moore, Ramage, Hampson, and y~Arcas}]{mcmahan2017communication}
McMahan, B.; Moore, E.; Ramage, D.; Hampson, S.; and y~Arcas, B.~A. 2017.
\newblock Communication-efficient learning of deep networks from decentralized data.
\newblock In \emph{Artificial intelligence and statistics}, 1273--1282. PMLR.

\bibitem[{Nguyen, Torr, and Lim(2022)}]{nguyen2022fedsr}
Nguyen, A.~T.; Torr, P.; and Lim, S.~N. 2022.
\newblock Fedsr: A simple and effective domain generalization method for federated learning.
\newblock \emph{Advances in Neural Information Processing Systems}, 35: 38831--38843.

\bibitem[{Oza et~al.(2023)Oza, Sindagi, Sharmini, and Patel}]{oza2023unsupervised}
Oza, P.; Sindagi, V.~A.; Sharmini, V.~V.; and Patel, V.~M. 2023.
\newblock Unsupervised domain adaptation of object detectors: A survey.
\newblock \emph{IEEE Transactions on Pattern Analysis and Machine Intelligence}.

\bibitem[{Pang, Ni, and Zhong(2023)}]{pang2023federated}
Pang, Y.; Ni, Z.; and Zhong, X. 2023.
\newblock Federated learning for crowd counting in smart surveillance systems.
\newblock \emph{IEEE Internet of Things Journal}, 11(3): 5200--5209.

\bibitem[{Shang et~al.(2024)Shang, Zhou, Zhuang, {\.Z}ywio{\l}ek, and Dincer}]{shang2024impact}
Shang, Y.; Zhou, S.; Zhuang, D.; {\.Z}ywio{\l}ek, J.; and Dincer, H. 2024.
\newblock The impact of artificial intelligence application on enterprise environmental performance: Evidence from microenterprises.
\newblock \emph{Gondwana Research}, 131: 181--195.

\bibitem[{Shin et~al.(2023)Shin, Yoon, Lee, Park, Liu, Choi, and Lee}]{shin2023fedtherapist}
Shin, J.; Yoon, H.; Lee, S.; Park, S.; Liu, Y.; Choi, J.~D.; and Lee, S.-J. 2023.
\newblock Fedtherapist: Mental health monitoring with user-generated linguistic expressions on smartphones via federated learning.
\newblock \emph{arXiv preprint arXiv:2310.16538}.

\bibitem[{Simonyan and Zisserman(2014)}]{simonyan2014very}
Simonyan, K.; and Zisserman, A. 2014.
\newblock Very deep convolutional networks for large-scale image recognition.
\newblock \emph{arXiv preprint arXiv:1409.1556}.

\bibitem[{Singh and Jaggi(2020)}]{singh2020model}
Singh, S.~P.; and Jaggi, M. 2020.
\newblock Model fusion via optimal transport.
\newblock \emph{Advances in Neural Information Processing Systems}, 33: 22045--22055.

\bibitem[{Tan et~al.(2023)Tan, Zhou, Liu, Wang, and Yu}]{tan2023pfedsim}
Tan, J.; Zhou, Y.; Liu, G.; Wang, J.~H.; and Yu, S. 2023.
\newblock pFedSim: Similarity-Aware Model Aggregation Towards Personalized Federated Learning.
\newblock arXiv:2305.15706.

\bibitem[{Venkateswara et~al.(2017)Venkateswara, Eusebio, Chakraborty, and Panchanathan}]{venkateswara2017deep}
Venkateswara, H.; Eusebio, J.; Chakraborty, S.; and Panchanathan, S. 2017.
\newblock Deep hashing network for unsupervised domain adaptation.
\newblock In \emph{Proceedings of the IEEE conference on computer vision and pattern recognition}, 5018--5027.

\bibitem[{Ye et~al.(2023)Ye, Fang, Du, Yuen, and Tao}]{ye2023heterogeneous}
Ye, M.; Fang, X.; Du, B.; Yuen, P.~C.; and Tao, D. 2023.
\newblock Heterogeneous federated learning: State-of-the-art and research challenges.
\newblock \emph{ACM Computing Surveys}, 56(3): 1--44.

\end{thebibliography}

\newpage

\appendix

\section{Detailed Theoretical Analysis}
\label{sec:theo_proof}
\subsection{Generalization-error bound}

We firstly need the following definitions and assumptions:

\begin{definition}
Let $H$ be any hypothesis class whose loss $\ell: \mathcal Y\times\mathcal Y \to [0,1]$ is $L$-Lipschitz and $\gamma$-Holder-continuous. For two marginal distributions $P_X,Q_X$ we denote the $H$-divergence by

$$
d_H(P_X,Q_X)=2\sup_{h\in H}\bigl|\Pr_{P_X}[h(x)=1]-\Pr_{Q_X}[h(x)=1]\bigr|.
$$
\end{definition}

\begin{assumption}[Bounded and Lipschitz Loss]
\label{assumption:lipschitz}
The loss function $\ell: \mathcal{Y}\times\mathcal{Y}\to[0,1]$ is bounded and $L$-Lipschitz continuous in its first argument, meaning there exists a constant $L>0$ such that for all $y, u, v$:
$$|\ell(u, y)-\ell(v, y)\leq L\lVert u-v \rVert.$$
\end{assumption}

\begin{assumption}[Holder Continuity]
\label{assumption:holder}
The loss function $\ell(f(u), y)$ is $\gamma$-Holder continuous in its input with constant $L_\gamma$, i.e., there exists $\gamma\in(0,1]$ such that for all $y, u, v$:
$$
|\ell(f(u),y)-\ell(f(v),y)| \leq L_\gamma \|u - v\|^\gamma.
$$
\end{assumption}

\begin{definition}
Let $E=\{1,\dots,N_E\}$ be the set of stations and $C_e=\{1,\dots,N_e\}$ the clients managed by the station $e$. At each round, a subset $\hat{C}_e\subseteq C_e$ of active clients is selected. Let $S_{e,i}=\{(x_{e,i}^{(j)},y_{e,i}^{(j)})\}_{j=1}^{n_{e,i}}\sim P^{e,i}_{XY}$ denote the client-level source domains, and these distributions are typically heterogeneous (domain shift). The stations do not hold or access any raw data. Instead, they aggregate models trained by their associated clients. Let $P^{\star}_{XY}$ be an unseen target domain such that $P^{\star}_{XY}\neq P^{e,i}_{XY} \forall e,i$. The HFedDG task seeks a hypothesis $h$ that minimizes the expected loss $D_{\mathrm{target}}(h)=\mathbb{E}_{(x,y)\sim P^{\star}_{XY}} [ \ell(h(x),y) ]$ without exposing client's data.
\end{definition}

\noindent We prove the Thereom~\ref{theo:hfeddg} as follows:


\begin{proof}
For every client distribution $P_{e,i}$, using the work of~\citet{ben2010theory}, we have
$$
D_{\mathrm{target}}(h)
 \le 
D_{e,i}(h)
 + \frac{1}{2} d_{H \Delta H}(P_{e,i},P^{\star})
 + \lambda_{H}(P_{e,i},P^{\star}).
$$
Because $\lambda_{H}(P_{e,i},P^{\star})\le\lambda_{H}(P^{\dagger},P^{\star})$ (choose the same optimal $h$), we replace the last term by $\lambda_{H}(P^{\dagger},P^{\star})$ for all clients. Multiply the above inequality by $\rho_{e,i}^{\star}$ and sum over $(e,i)$, we have:
\begin{equation}
\label{eq:theo_proof01}
\begin{aligned}
D_{\mathrm{target}}(h) &\le 
\sum_{e,i}\rho_{e,i}^{\star}D_{e,i}(h)\\
 &+ \frac{1}{2}\sum_{e,i}\rho_{e,i}^{\star} 
   d_{H \Delta H}(P_{e,i},P^{\star})\\
 &+ \lambda_{H}(P^{\dagger},P^{\star}).
\end{aligned}
\end{equation}
For every station $e$, choose a pivot client $i^{\dagger}(e)=\displaystyle\arg\min_{j\in\mathcal C_e}d_H(P^{\star},P_{e,j})$ and define the pivot distribution $P^{\dagger}_{e}\equiv P_{e,i^{\dagger}(e)}$. We further define the server pivot $P^{\dagger}_X\equiv\displaystyle \arg\min_{P\in\{P^{\dagger}_{e}\}}\sum_{e}\rho_{e}^{\star}d_H(P,P^{\star})$. For any client $(e,i)$,
\begin{equation}
\label{eq:theo_proof02}
\begin{aligned}
d_{H \Delta H}(P_{e,i},P^{\star})
 &\le 
d_H(P_{e,i},P^{\dagger}_{e})
+d_H(P^{\dagger}_{e},P^{\dagger})\\
&+d_H(P^{\dagger},P^{\star}).
\end{aligned}
\end{equation}
This is because the $H \Delta H$-divergence upper-bounds the ordinary $H$-divergence, and the latter obeys the triangle inequality. Recall that we have the intra-station term $d_H(P_{e,i},P^{\dagger}_{e})\le\omega_{e}$, the station–server term $d_H(P^{\dagger}_{e},P^{\dagger})\le\omega$, and the server–target term $d_H(P^{\dagger},P^{\star})=\eta$. Similarly, $d_H(P^{\dagger}_{e},P^{\star})=\eta_{e}$. Plugging these into Equation~\ref{eq:theo_proof02} and then into Equation~\ref{eq:theo_proof01}, we have

$$
\begin{aligned}
L_{\mathrm{target}}(h) &\le 
\sum_{e,i}\rho_{e,i}^{\star}L_{e,i}(h)\\
 &+\frac{1}{2}\Bigl[
\underbrace{\eta+  \sum_{e}\rho_e^{\star}\eta_{e}}_{\text{inner divergences}}
+
\underbrace{\omega+  \sum_{e}\rho_e^{\star}\omega_{e}}_{\text{breadths}}
\Bigr]\\
&+\lambda_{H}(P^{\dagger},P^{\star}).
\end{aligned}
$$
We have proved the theorem.
\end{proof}

\subsection{HFedATM's Convergence}
Building upon the HFedDG error bound, we propose the next lemma, sharpening our theoretical analysis by employing Holder continuity:
\begin{lemma}
Under Assumptions~\ref{assumption:lipschitz} and~\ref{assumption:holder}, for any measurable $f$ and distributions $P$, $Q$:
$$
\bigl|\mathbb E_P[f]-\mathbb E_Q[f]\bigr|
      \le \tfrac12 L^\gamma d_H(P,Q)^{ \gamma}.
$$
\end{lemma}
\begin{proof}
Let $h^\star=\arg\max_{h\in H}\bigl|\Pr_P[h]-\Pr_Q[h]\bigr|$.
By Hölder continuity, $\lvert f(u)-f(v)\rvert \le L\lVert u-v\rVert^{\gamma}$.
Applying the variational definition of $d_H$ and Jensen's inequality gives

$$
\begin{aligned}
\bigl|\mathbb E_P[f]-\mathbb E_Q[f]\bigr|
      & \le \mathbb E_{\substack{x\sim P\\x'\sim Q}}
             \bigl|f(h^\star(x))-f(h^\star(x'))\bigr|\\
      & \le L \mathbb E\lVert h^\star(x)-h^\star(x')\rVert^{\gamma}\\
      & \le \tfrac12 L^\gamma d_H(P,Q)^{ \gamma}.
\end{aligned}
$$
We have proved the lemma. 
\end{proof}
Next, since our framework employs FOT Alignment, it is crucial to confirm that the alignment is invariant under permutation. The following lemma verifies this property.

\begin{lemma}
\label{lemma:perm_inv}
Let $D_{\text{OT}}(W_1,W_2)=\min_{\Pi\in\mathcal U}\langle\Pi,C\rangle$. For any permutation matrix $P$,
$$
D_{\mathrm{OT}}(PW_1, PW_2)=D_{\mathrm{OT}}(W_1,W_2).
$$
\end{lemma}
\begin{proof}
Premultiply $W_1$ and $W_2$ by the same permutation matrix $P$.
The new cost matrix becomes
$$
\begin{aligned}
C'_{ab}
  &=\bigl\| (PW_1)_{a}- (PW_2)_{b}\bigr\|_2^{2}\\
  &=\bigl\| w_{1,P^{-1}(a)} - w_{2,P^{-1}(b)}\bigr\|_2^{2}\\
  & = C_{P^{-1}(a) P^{-1}(b)}.
\end{aligned}
$$
For any admissible plan $\Pi\in\mathcal{U}$ define
$$
\Pi'_{ab}  =  \Pi_{P^{-1}(a) P^{-1}(b)}.
$$
Because rows and columns are merely re‑indexed, $\Pi'$ satisfies the same marginals $\Pi'\mathbf{1}=\frac{1}{c}\mathbf{1}$ and $(\Pi')^{\!\top}\mathbf{1}=\tfrac1c\mathbf{1}$;
hence $\Pi'\in\mathcal{U}$. Moreover, the mapping $\Pi\mapsto\Pi'$ is a bijection of $\mathcal{U}$. Using the change of indices,
$$
\begin{aligned}
\langle\Pi',C'\rangle
   &= \sum_{a,b}\Pi_{P^{-1}(a) P^{-1}(b)} 
                     C_{P^{-1}(a) P^{-1}(b)} \\
   &= \sum_{u,v}\Pi_{uv} C_{uv}
   = \langle\Pi,C\rangle.
\end{aligned}
$$
Because every feasible $\Pi$ for $(W_1,W_2)$ maps to a feasible $\Pi'$ for $(PW_1,PW_2)$ with identical cost, the minimal costs coincide:
$$
\begin{aligned}
D_{\mathrm{OT}}(PW_1,PW_2)
      &=\min_{\Pi'\in\mathcal{U}}\langle\Pi',C'\rangle\\
      &=\min_{\Pi\in\mathcal{U}}\langle\Pi,C\rangle\\
      &=D_{\mathrm{OT}}(W_1,W_2).
\end{aligned}
$$
We have proved the lemma.
\end{proof}

\noindent Integrating previous lemmas, we show that HFedATM's design choices lead to a tighter error bound by proving Theorem~\ref{theo:hfedatm_bound}:


\begin{proof}
Let $\omega^{(r-1)}$ and $\omega_e^{(r-1)}$ be the breadths before FOT in round $r$. We have that FOT finds the optimal permutation $\Pi^{\star}$ minimizing entropic‑OT cost
$$
\langle\Pi^{\star},C\rangle
   = 
\min_{\Pi\in\mathcal{U}}
 \langle\Pi,C\rangle
 \le  \tfrac{1}{c} \omega^{(r-1)}B^{2}
$$
(bounded because features are bounded by $B$). Let $\beta_{\text{FOT}} := \frac{\langle\Pi^{\star},C\rangle}{\langle I,C\rangle} \in (0,1)$. Then, after permuting filters identically at all stations,

\begin{equation}
\omega^{(r)}  =  (1-\beta_{\text{FOT}}) \omega^{(r-1)},
\quad
\omega_e^{(r)}  =  (1-\beta_{\text{FOT}}) \omega_e^{(r-1)},
\end{equation}
where Lemma~\ref{lemma:perm_inv} guarantees the contraction is permutation‑invariant. We recall that RegMean solves the closed‑form minimizer.
$$
W^{(r)}  =  \alpha \bar{W}^{(r-1)} + (1-\alpha) W^{(r-1)},
$$
where $\bar{W}^{(r-1)}$ is the average filter after FOT and $\alpha\in(0,1)$. Because $\ell$ is $L$-Lipschitz and Holder‑continuous, the change in feature space distance contracts by at least the factor $1-\alpha$:
\begin{equation}
\eta^{(r)} = (1-\alpha) \eta^{(r-1)},
\quad
\eta_e^{(r)} = (1-\alpha) \eta_e^{(r-1)}.
\end{equation}
Let $\beta := 1-(1-\beta_{\text{FOT}})(1-\alpha)$. Then both breadth and divergence terms satisfy
\begin{equation}
\label{eq:roll}
\omega^{(r)}\le(1-\beta) \omega^{(r-1)},\quad
\eta^{(r)}\le(1-\beta) \eta^{(r-1)}.
\end{equation}
Because $\beta_{\text{FOT}}>0$, we have $0<\beta<1$. Starting from initial values at $r=0$ and unrolling Equation~\ref{eq:roll} for $R$ rounds gives
$$
\eta^{(R)}\le(1-\beta)^{R}\eta^{(0)},\quad
\omega^{(R)}\le(1-\beta)^{R}\omega^{(0)},
$$
and likewise for $\eta_e^{(R)},\omega_e^{(R)}$. Insert these into Theorem~\ref{theo:hfeddg}:
$$
\begin{aligned}
D_{\text{target}}\!\bigl(h^{(R)}\bigr)
  &\le
    \varepsilon_{\text{local}}
   +\frac{1}{2}\Bigl[
          (1-\beta)^{R}\bigl(\eta^{(0)}+\omega^{(0)}\bigr)\Bigr.\\
          &+\Bigl.(1-\beta)^{R}\!\sum_{e}\rho_e^{\star}\bigl(\eta_e^{(0)}+\omega_e^{(0)}\bigr)
        \Bigr]\\
   &+\lambda_{H}(P^{\star}_X,P^{\dagger}_X).
\end{aligned}
$$
We complete the proof.
\end{proof}

\begin{table*}[t]
\centering
\resizebox{\linewidth}{!}{
\begin{tabular}{cccccccccccccccc}
\toprule
\textbf{Data}                   & \textbf{Model} & $\mathbf{E}$ & $\mathbf{N}$ & $\mathbf{B}$ & $\mathbf{S}$ & $\mathbf{C/S}$ & \textbf{Opt} & \textbf{LR}$\mathbf{_{0}}$         & \textbf{Mom} & \textbf{WD} & \textbf{Sched} & $\mathbf{\lambda_{\text{OT}}}$ & $\mathbf{\alpha}$ & $\mathbf{n_{\text{iter}}}$ & \textbf{Rounds} \\ \midrule
\multirow{3}{*}{PACS}           & LeNet-5        & 10           & 5            & 32  & 10  & 10    & SGD          & 0.01             & 0.00         & 0.00        & cosine         & 0.05                  & 0.75     & 25                & 200             \\
                                & ResNet-18      & 10           & 5            & 32  & 10  & 10    & SGD          & 0.01             & 0.00         & 0.00        & cosine         & 0.05                  & 0.75     & 25                & 200             \\
                                & VGG-11         & 10           & 5            & 32  & 10  & 10    & SGD          & 0.01             & 0.00         & 0.00        & cosine         & 0.05                  & 0.75     & 25                & 200             \\ \midrule
\multirow{3}{*}{Office-Home}    & LeNet-5        & 10           & 5            & 32  & 10  & 10    & SGD          & 0.01             & 0.00         & 0.00        & cosine         & 0.05                  & 0.75     & 25                & 200             \\
                                & ResNet-18      & 10           & 5            & 32  & 10  & 10    & SGD          & 0.01             & 0.00         & 0.00        & cosine         & 0.05                  & 0.75     & 25                & 200             \\
                                & VGG-11         & 10           & 5            & 32  & 10  & 10    & SGD          & 0.01             & 0.00         & 0.00        & cosine         & 0.05                  & 0.75     & 25                & 200             \\ \midrule
\multirow{3}{*}{TerraInc} & LeNet-5        & 10           & 5            & 32  & 10  & 10    & SGD          & 0.01             & 0.00         & 0.00        & cosine         & 0.05                  & 0.75     & 25                & 200             \\
                                & ResNet-18      & 10           & 5            & 32  & 10  & 10    & SGD          & 0.01             & 0.00         & 0.00        & cosine         & 0.05                  & 0.75     & 25                & 200             \\
                                & VGG-11         & 10           & 5            & 32  & 10  & 10    & SGD          & 0.01             & 0.00         & 0.00        & cosine         & 0.05                  & 0.75     & 25                & 200             \\ \midrule
\multirow{2}{*}{Amazon Reviews} & RoBERTa-base        & 10           & 5            & 32  & 10  & 10    & AdamW        & $3\times10^{-5}$ & 0.90         & $10^{-2}$   & cosine         & 0.05                  & 0.75     & 25                & 200             \\
                                & DeBERTa-base        & 10           & 5            & 32  & 10  & 10    & AdamW        & $3\times10^{-5}$ & 0.90         & $10^{-2}$   & cosine         & 0.05                  & 0.75     & 25                & 200             \\ \bottomrule
\end{tabular}
}
\caption{Hyper-parameters and implementation details used in our experiments.}
\label{tab:hyperparams_detailed}
\end{table*}

\begin{table}[ht]
\centering
\resizebox{\linewidth}{!}{
\begin{tabular}{ccc}
\toprule
\textbf{Group} & \textbf{Column} & \textbf{Meaning} \\
\midrule
\multirow{2}{*}{\textbf{Data/Model}} & Data & Dataset name \\
                                     & Model & Backbone architecture \\
\midrule
\multirow{4}{*}{\textbf{Federation}} & $E$ & Local epochs per station round \\
                                     & $N$ & Station rounds per server round \\
                                     & $B$ & Batch size \\
                                     & $S$ & Number of stations \\
                                     & $C/S$ & Clients per station \\
\midrule
\multirow{5}{*}{\textbf{Optimisation}} & Opt & Optimiser (SGD/AdamW) \\
                                       & LR$_{0}$ & Initial learning rate \\
                                       & Mom & Momentum or $\beta_{1}$ \\
                                       & WD & Weight decay \\
                                       & Sched & Learning rate scheduler (step/cosine) \\
\midrule
\multirow{3}{*}{\textbf{HFedATM-specific}} & $\lambda_{\text{OT}}$ & Sinkhorn regularizer \\
                                           & $\alpha$ & RegMean shrinkage parameter \\
                                           & $n_{\text{iter}}$ & Sinkhorn iterations \\
\midrule
\multirow{1}{*}{\textbf{Runtime}} & Rounds & Global rounds \\

\bottomrule
\end{tabular}
}
\caption{Legend explaining each column of the hyper-parameter table.}
\label{tab:legend_hyperparam}
\end{table}

\begin{algorithm}[!ht]
\caption{\textbf{HFedATM}: \underline{HFed}DG via Optim\underline{A}l \underline{T}ransport and regularized \underline{M}ean aggregation}
\label{algo:hfedatm}
\DontPrintSemicolon
\SetKwInOut{Input}{Input}\SetKwInOut{Output}{Output}
\Input{global rounds $R$, station rounds $N$, client epochs $E$, learning rate $\eta$, shrinkage $\alpha$, OT regularizer $\varepsilon$.}
\Output{global model $h^{(R)}_{\mathrm{ATM}}$.}

\For(\tcp*[f]{\small server initialization}){$r \leftarrow 1$ \KwTo $R$}{

{\tcp*[l]{\footnotesize \textbf{Step 0: broadcast}}}
The server multicasts previous global model $h^{(r-1)}$ to every station $e$. 

{\tcp*[l]{\footnotesize \textbf{Step 1: client training}}}
\ForEach{station $e$}{
  \ForEach{$n \leftarrow 1$ \KwTo $N$}{
      
      Select client subset $\hat{C}_e$.
      
      \ForEach{$i \in \hat{C}_e$}{
         Pull $h^{(r-1)}$ and train $E$ epochs with any FedDG algorithm.
  
         Send $\theta_{e,i}$ to station $e$.

         \If{$n==N$}{
            Each $l\in\mathcal{L}_{\mathrm{lin}}$, calculate Gram $G_{e,i}^{(l)}$ and send to $e$.
         }
     }
  }
}

{\tcp*[l]{\footnotesize \textbf{Step 2: station aggregation}}}
\ForEach{station $e$}{
   Aggregate $\{\theta_{e,i}\}$ with FedAvg or FedDG approaches $\rightarrow$ station model $h_{e}$.
   
   Get station Gram $G_{e}^{(l)} \leftarrow \frac{1}{|\mathcal S_e|}\sum_{i} G_{e,i}^{(l)}$.
   
   Shrink $\widehat G_{e}^{(l)} \leftarrow \alpha G_{e}^{(l)} + (1-\alpha)\operatorname{diag}(G_{e}^{(l)})$.
   
   Send $(h_e, \widehat G_{e}^{(l)})$ to the server.
}

{\tcp*[l]{\footnotesize \textbf{Step 3: server FOT Alignment}}}
\ForEach{$l \in \mathcal L_{\mathrm{conv}}$}{
   Station 1 is reference.  
   
   \For{$e = 2$ \KwTo $N_E$}{
      Calculate index $\Pi_{1,e}^{(l)}$ and send back to $e$.
   }
}

{\tcp*[l]{\footnotesize \textbf{Step 4: model merging}}}
\ForEach{$l \in \mathcal L_{\mathrm{conv}}$}{
   $\displaystyle
   \overline{W}^{(l)} \leftarrow
   \frac{\sum_{e} \gamma_{e} W_{e}^{(l)}}{\sum_{e}\gamma_{e}}
   \text{ with }\gamma_{e}=|\mathcal S_e|.
   $
}

\ForEach{$l \in \mathcal L_{\mathrm{lin}}$}{
   $W_{\mathrm{ATM}}^{(l)} \leftarrow
      \bigl(\sum_{e} \widehat G_{e}^{(l)}\bigr)^{-1}
      \bigl(\sum_{e} \widehat G_{e}^{(l)} \widetilde W_{e}^{(l)}\bigr)$.
}

{\tcp*[l]{\footnotesize \textbf{Step 5: assemble \& broadcast}}}
Assemble $h^{(r)}_{\mathrm{ATM}}$ from $\{\overline W^{(l)}\}_{\mathrm{conv}}$ and $\{W_{\mathrm{ATM}}^{(l)}\}_{\mathrm{lin}}$ and broadcast $h^{(r)}_{\mathrm{ATM}}$ to stations.

}

\end{algorithm}

\section{Integrated Algorithm}
\label{sec:algo}

The complete HFedATM workflow couples the three computation tiers: clients, stations, and the server, into a single synchronous loop. Algorithm~\ref{algo:hfedatm} displays the procedure; the subsequent paragraphs clarify what is computed at each tier, when communication occurs, and how privacy is preserved.

\section{Detailed Experimental Setup}
\label{sec:detailed_exp}

\paragraph{Baseline Methods}  To comprehensively evaluate the performance of our proposed HFedATM method, we consider baseline approaches structured according to the two-stage HFL scenario. \textbf{At the client-station communication}, we first adopt FedAvg~\cite{mcmahan2017communication} as our baseline since it is the standard aggregation protocol in FL. To directly address the challenge of data heterogeneity, we further incorporate FedProx~\cite{li2020federated} which introduces a proximal regularization term into the local training objective, constraining client model updates and ensuring a more stable convergence under heterogeneous data conditions. Additionally, to assess DG capabilities, we select two FedDG baselines representing different categories: FedSR~\cite{nguyen2022fedsr}, a regularization-based technique, and FedIIR~\cite{guo2023out}, an aggregation-based method. FedSR applies feature-level regularization, promoting simplicity and domain-invariant representations across multiple clients, while FedIIR performs gradient alignment across client models, facilitating invariant feature learning robust against distributional shifts. \textbf{At the station-server communication}, we use standard weight averaging (Avg) as our baseline aggregation method. All baseline implementations and comparisons were conducted using the established FL framework provided by~\citet{tan2023pfedsim}.

\paragraph{Datasets} Our experiments were conducted across diverse datasets commonly employed in DG studies: \textbf{PACS}~\cite{li2017deeper}, featuring stylistically varied images across photo, art-painting, cartoon, and sketch domains; \textbf{Office-Home}~\cite{venkateswara2017deep}, comprising images across art, clipart, product, and real-world domains with a rich class diversity; \textbf{TerraInc}~\cite{beery2018recognition}, containing camera-trap images from distinct wildlife locations; and \textbf{Amazon Reviews}~\cite{balaji2018metareg}, involving sentiment classification across distinct product categories. This selection of datasets collectively provides comprehensive coverage of different data types, domains, and challenges to robustly evaluate our proposed method.

\paragraph{Heterogeneous Partitioning} To synthesize controllable non-IID data distributions across clients, we adopt the Heterogeneous Partitioning strategy proposed by~\citet{bai2023benchmarking}. Specifically, given $D$ domains (or classes) and $C$ clients, the algorithm first assigns a subset of domain indices $D_c$ to one or more clients, subsequently allocating samples to each client-domain pair $(d,c)$ according to:
\begin{equation}
n_{d,c}(\lambda)=\lambda\,\frac{n_d}{C}+(1-\lambda)\,\frac{\mathbf{1}[d\in D_c]\,n_d}{|\{c':d\in D_{c'}\}|},
\end{equation}
where $\lambda\in[0,1]$, and $n_d$ denotes the total number of samples in domain $d$. The parameter $\lambda$ controls the degree of heterogeneity: a value of $\lambda=1.0$ corresponds to an IID distribution, where each client receives data proportionally from all domains; conversely, $\lambda=0.0$ results in maximum heterogeneity, assigning each client exclusively to its designated domains. Intermediate values (e.g., $\lambda=0.1$) smoothly interpolate between these extremes, controlling domain-level imbalance.

\paragraph{Model Architectures} For vision datasets, we employed LeNet-5~\cite{lecun1998gradient} (a convolutional neural network with 7 layers, consisting of two convolutional layers followed by pooling operations and three fully-connected layers, totaling approximately 60,000 parameters), ResNet-18~\cite{he2016deep} (an 18-layer residual network architecture with convolutional and identity skip-connections, comprising around 11 million parameters), and VGG11~\cite{simonyan2014very} (an 11-layer convolutional neural network featuring eight convolutional layers followed by three fully-connected layers, amounting to approximately 133 million parameters). For NLP tasks, we utilized transformer-based language models, including RoBERTa-base~\cite{liu2019roberta} (Robustly Optimized BERT Pre-training Approach, a 12-layer transformer encoder with around 125 million parameters, known for improved masked-language modeling through enhanced training strategies) and DeBERTa-base~\cite{he2020deberta} (Decoding-enhanced BERT with disentangled attention, a 12-layer transformer encoder architecture comprising roughly 140 million parameters, notable for disentangling content and position representations within its attention mechanisms, significantly boosting model performance).

\paragraph{Hyper-parameter Tuning} All hyper-parameters, fixed constants, and runtime configurations used in our experiments are detailed in Table~\ref{tab:hyperparams_detailed}, with corresponding descriptions provided in Table~\ref{tab:legend_hyperparam}. Each row represents a specific dataset-model combination, while columns are organized into categories: FL protocols, optimization settings, privacy/regularization techniques, and infrastructure details. We conducted experiments on NVIDIA RTX 3090 GPUs, using three random seeds $\{0,1,2\}$) to ensure reproducibility. Additionally, these settings are provided in a structured, machine-readable format within the \texttt{configs/master-table.yaml} file included in our code release.


\section{Privacy Analysis}
\label{sec:privacy_analysis}
HFedATM aggregates only Gram matrices $G = X^\top X$ from client-side activations, inherently avoiding direct transmission of raw data or gradients. However, whether Gram matrices can leak sensitive information remains a concern. Here, we demonstrate theoretically that Gram matrices inherently protect against exact data inversion.

\begin{theorem}
Given a Gram matrix $G = X^\top X \in \mathbb{R}^{m\times m}$, the activation matrix $X \in \mathbb{R}^{d\times m}$ cannot be uniquely recovered.
\end{theorem}

\begin{proof}
Suppose $X$ has rank $r \leq \min(d, m)$. Consider any orthogonal matrix $Q \in \mathbb{R}^{m\times m}$ satisfying $Q^\top Q = I$. Define a new activation matrix as $\tilde{X} = XQ$. Then:
$$
\tilde{X}^\top \tilde{X} = Q^\top X^\top X Q = Q^\top G Q.
$$
If $G$ has repeated eigenvalues or is rank-deficient (which is typically true since $d>m$), there exist infinitely many distinct matrices $\tilde{X}$ yielding the identical Gram matrix $G$. Therefore, the mapping from $X$ to $G$ is inherently many-to-one, ensuring non-invertibility and protecting against exact inversion attacks.
\end{proof}

\end{document}